\def\year{2022}\relax
\documentclass[letterpaper]{article} 
\usepackage{aaai22}  
\usepackage{times}  
\usepackage{helvet}  
\usepackage{courier}  
\usepackage[hyphens]{url}  
\usepackage{graphicx} 
\usepackage{natbib}  
\usepackage{caption} 
\DeclareCaptionStyle{ruled}{labelfont=normalfont,labelsep=colon,strut=off} 
\usepackage{algorithm}
\usepackage{algorithmic}
\usepackage[utf8]{inputenc} 
\usepackage[T1]{fontenc}    
\usepackage{url}            
\usepackage{booktabs}       
\usepackage{amsfonts}       
\usepackage{nicefrac}       
\usepackage{microtype}      
\usepackage{lipsum}
\usepackage{amsmath, amsthm, amssymb, amsfonts}
\usepackage{multirow}
\usepackage{subfigure}
\usepackage{color}
\usepackage{enumitem}
\usepackage{enumerate}
\usepackage{textcomp}
\usepackage{microtype}
\usepackage{booktabs} 
\usepackage{times}
\usepackage{soul}
\usepackage{url}
\usepackage{enumerate}
\usepackage{dsfont}
\usepackage{multicol}
\usepackage{xcolor}
\newtheorem{theorem}{Theorem}

\newtheorem{lemma}{Lemma}
\usepackage[switch]{lineno}  %
\usepackage{hyperref} 

\usepackage{newfloat}
\usepackage{listings}
\floatstyle{ruled}
\newfloat{listing}{tb}{lst}{}
\floatname{listing}{Listing}
\title{Offline Reinforcement Learning with Soft Behavior Regularization}
\author {
    Haoran Xu\textsuperscript{\rm 1},
    Xianyuan Zhan\textsuperscript{\rm 2},
    Jianxiong Li\textsuperscript{\rm 2},
    Honglei Yin\textsuperscript{\rm 1}
}
\affiliations {
    \textsuperscript{\rm 1} JD Intelligent Cities Research, JD Technology, Beijing, China\\
    \textsuperscript{\rm 2} Institute for AI Industry Research (AIR), Tsinghua University, Beijing, China\\
    \{ryanxhr, zhanxianyuan, yinhonglei93\}@gmail.com, li-jx21@mails.tsinghua.edu.cn
}

\begin{document}


\maketitle

\begin{abstract}
Most prior approaches to offline reinforcement learning (RL) utilize \textit{behavior regularization}, typically augmenting existing off-policy actor critic algorithms with a penalty measuring divergence between the policy and the offline data.
However, these approaches lack guaranteed performance improvement over the behavior policy. In this work, we start from the performance difference between the learned policy and the behavior policy, we derive a new policy learning objective that can be used in the offline setting, which corresponds to the advantage function value of the behavior policy, multiplying by a state-marginal density ratio.
We propose a practical way to compute the density ratio and demonstrate its equivalence to a state-dependent behavior regularization. Unlike state-independent regularization used in prior approaches, this \textit{soft} regularization allows more freedom of policy deviation at high confidence states, leading to better performance and stability. We thus term our resulting algorithm Soft Behavior-regularized Actor Critic (SBAC).
Our experimental results show that SBAC matches or outperforms the state-of-the-art on a set of continuous control locomotion and manipulation tasks.
\end{abstract}

\section{Introduction}
Reinforcement learning (RL) is an increasingly important technology for developing highly capable AI systems, it has achieved great success in domains like games~\cite{mnih2013playing}, recommendation systems~\cite{choi2018reinforcement}, and robotics~\cite{levine2016end}, etc. However, the fundamental \textit{online} learning paradigm in RL is also one of the biggest obstacles to RL's widespread adoption, as interacting with the environment can be costly and dangerous in real-world settings. Furthermore, even in domains where online interactions are feasible, we might still prefer to utilize previously collected \textit{offline} data, as online data collection itself is shown to cause poor generalization~\cite{kumar2020discor}.

Offline Reinforcement Learning, also known as batch RL or data-driven RL, aims at solving the abovementioned problems by learning effective policies solely from offline static data, without any additional online interactions.
This setting offers the promise of reproducing the most successful formula in Machine Learning (e.g., CV, NLP), where we combine large and diverse datasets (e.g. ImageNet) with expressive function approximators to enable effective generalization in complex real-world tasks~\cite{zhan2021deepthermal}.

In contrast to offline RL, off-policy RL uses a replay buffer that stores transitions that are actively collected by the policy throughout the training process. Although off-policy RL algorithms are considered able to leverage any data to learn skills, past work has shown that they cannot be directly applied to static datasets, due to the extrapolation error of the Q-function caused by out-of-distribution actions~\cite{fujimoto2019off}, and the error can not be eliminated without requiring growing batch of online samples.

Prior works tackle this problem by ensuring that the learned policy stays "close" to the behavior policy via behavior regularization. This is achieved either by explicit constraints on the learned policy to only select actions where $(s, a)$ has sufficient support under the behavior distribution~\cite{fujimoto2019off,ghasemipour2021emaq}; or by adding a regularization term that calculates some divergence metrics between the learned policy and the behavior policy~\cite{wu2019behavior,siegel2019keep,zhang2020brac+,dadashi2021offline}, e.g., KL divergence~\cite{jaques2020human} or Maximum Mean Discrepancy (MMD)~\cite{kumar2019stabilizing}. While straightforward, these methods lack guaranteed performance improvement against the behavior policy.
They also require an accurate $Q^{\pi}$, which is hard to estimate due to the abovementioned distribution shift issue and the iterative error exploitation issue caused by dynamic programming.

In this work, we propose an alternative approach, we start from a conceptual derivation from the perspective of guaranteed policy improvement against the behavior policy. However, the derived objective requires states sampled from the learned policy, which is impossible to get in the offline setting. We then derive a new policy learning objective that can be used in the offline setting, which corresponds to the advantage function value of the behavior policy, multiplying by a discounted marginal state density ratio. 
We propose a practical way to compute the density ratio and demonstrate its equivalence to a state-dependent behavior regularization. Unlike state-independent regularization used in prior approaches, our regularization term is softer.
It allows more freedom of policy deviation at high confidence states, leading to better performance, it also alleviates the distributional shift problem, making the learned policy more stable and robust when evaluating in the testing environment.

We thus term our resulting algorithm Soft Behavior-regularized Actor Critic (SBAC).
We present an extensive evaluation of SBAC on standard MuJoCo benchmarks and Adroit hand tasks with human demonstrations. We find that SBAC achieves state-of-the-art performance compared to a variety of existing model-free offline RL methods. We further show that while SBAC learns better policies, it also achieves much more stable and robust results during the online evaluation~\cite{fujimoto2021minimalist}.

\section{Background}
In this section, we introduce the notation and assumptions used in this paper, as well as provide a review of related methods in the literature.

\subsection{Reinforcement Learning}
We consider the standard fully observed Markov Decision Process (MDP) setting~\cite{sutton1998introduction}. 
An MDP can be represented as $\mathcal{M}=(\mathcal{S}, \mathcal{A}, P, r, \rho, \gamma)$ where $\mathcal{S}$ is the state space, $\mathcal{A}$ is the action space, $P(\cdot | s, a)$ is the transition probability distribution function, $r(s, a)$ is the reward function, $\rho$ is the initial state distribution and $\gamma$ is the discount factor, we assume $\gamma \in (0, 1)$ in this work. The goal of RL is to find a policy $\pi(\cdot | s)$ that maximizes the expected cumulative discounted reward starting from $\rho$ as 
\begin{align*}
\max_{\pi} \underset{ \substack{ s_{0} \sim \rho, a_{t} \sim \pi\left(\cdot | s_{t} \right) \\ s_{t+1} \sim P\left(\cdot | s_{t}, a_{t} \right) } } {\mathbb{E}}\left[\sum_{t+1}^{\infty} \gamma^{t} r\left( s_{t}, a_{t} \right)\right].
\end{align*}

In the commonly used actor critic paradigm, one optimizes a policy $\pi_{\theta}(\cdot | s)$ by alternatively learning a Q-function $Q_{\psi}$ to minimize squared Bellman evaluation errors over single step transitions $\left( s, a, r, s' \right)$ as 
\begin{align}
\label{eq_q_evaluation}
J(Q) = \underset{ \substack{ s,a,r,s' \sim \mathcal{B} \\ a' \sim \pi (\cdot | s') } } {\mathbb{E}} \left[\left( r+\gamma Q_{\psi'}\left(s', a'\right) - Q_{\psi}(s, a) \right)^{2}\right],
\end{align}
where $Q_{\psi'}$ denotes a target Q-function, which is periodicly synchronized with the current Q function.
Then, the policy is updated to maximize the Q-value, $\mathbb{E}_{a \sim \pi(\cdot \mid s)}\left[Q_{\psi}(s, a)\right]$.

\subsection{Offline Reinforcement Learning}
\label{sec_offline_rl}
In this work, we focus on the offline setting, in which the
agent cannot generate new experience data and the goal is to learn from a fixed dataset $\mathcal{B}=\left\{\left(s_{i}, a_{i}, s_{i}^{\prime}, r_{i}\right)\right\}_{i=1}^{N}$ consisting of single step transitions $\left\{\left(s_{i}, a_{i}, s_{i}^{\prime}, r_{i}\right)\right\}$. The dataset is assumed to be collected using an unknown behavior policy $\mu$ which denotes the conditional distribution $p(a|s)$ observed in the dataset. Note that $\pi_b$ may be multimodal distributions. 

We do not assume complete state-action coverage of the offline dataset, this is contrary to the common assumption in batch RL literature \cite{lange2012batch}, but is more realistic to the real-world setting. Under this setting, standard reinforcement learning methods such as SAC~\cite{haarnoja2018soft} or DDPG~\cite{lillicrap2016continuous} suffer when applied to these datasets due to extrapolation errors~\cite{fujimoto2019off,kumar2019stabilizing}.

To avoid such cases, behavior regularization is adopted to force the learned policy to stay close to the behavior policy~\cite{fujimoto2019off,kumar2019stabilizing,wu2019behavior}, these approaches share a similar framework, given by
\begin{align}
\label{eq_behavior_reg}
\max_{\pi} \underset{s \sim \mathcal{B}}{\mathbb{E}} \left[ \underset{a \sim \pi(\cdot|s)}{\mathbb{E}} \left[ Q^{\pi}(s,a) \right] - \alpha D(\pi(\cdot|s), \mu(\cdot|s)) \right],
\end{align}
where $D$ is some divergence measurement.

Although these approaches show promise, none provide improvement guarantees relative to the behavior policy and they share some common problems. 
The first problem is that these methods require an accurately estimated $Q^{\pi}$. However, this is hard to achieve due to the distribution shift between the behavior policy and the policy to be evaluated, which will introduce evaluation errors in Equation (\ref{eq_q_evaluation})~\cite{levine2020offline}. Furthermore, these errors will be accumulated and propagated across the state-action space through the iterative dynamic programming procedure.
The second problem is that, even though we can estimate an accurate Q function, the behavior regularization term may be too restrictive, which will hinder the performance of the learned policy. An ideal behavior regularization term should be state-dependent~\cite{sohn2020brpo}. This will make the policy less conservative, exploit large policy changes at high confidence states without risking poor performance at low confidence states. This also holds more theoretical guarantee~\cite{lee2020batch} which is typically missing in current approaches.

\section{Soft Behavior-regularized Actor Critic}
We now continue to describe our approach for behavior regularization in offline RL settings, which aims to circumvent the following issues associated with competing methods described in the previous section, including \textbf{(1)} lack policy improvement guarantee; \textbf{(2)} hard to estimate $Q^{\pi}$ due to distributional shift; \textbf{(3)} have too much conservatism due to state-independent regularization weight. We begin with a conceptual derivation of our method from the perspective of guaranteed policy improvement against the behavior policy. We then demonstrate our method from a different perspective, we show the equivalence of our method to a kind of soft divergence regularization, this soft regularization considers the marginal-state visitation difference and is state-dependent.

\subsection{Policy Improvement Derivation}
\label{sec_improvement_derivation}
We start from the difference in performance between the two policies. The following lemma allows us to compactly express the difference using the discounted marginal state visitation distribution, $d^{\pi}$, defined by $d^{\pi}(s) = (1-\gamma) \sum_{t=0}^{\infty} \gamma^t P(s_t=s|\pi)$.
\begin{lemma} [Performance difference~\cite{kakade2002approximately}] For any two policy $\pi$ and $\mu$,
\begin{align}
\label{eq_performance_diff}
\Delta(\pi, \mu) = J(\pi) - J(\mu) &= \frac{1}{1-\gamma} \ \underset{\substack{s \sim d^{\pi} \\ a \sim \pi}}{\mathbb{E}} \left[ A^{\mu}(s,a) \right]. 
\end{align}
\end{lemma}

This lemma implies that maximizing Equation (\ref{eq_performance_diff}) will yield a new policy $\pi$ with guaranteed performance improvement over a given policy $\mu$. 
Unfortunately, this objective cannot be used in the offline setting, as it requires access to on-policy samples from $d^{\pi}$, which is only available by interacting with the environment.

One way to mitigate this issue, and has been largely used in on-policy RL literature, is to assume that policy $\pi$ and $\mu$ give rise to similar state visitation distributions if they are close to each other, in other words, $\pi$ is in the ``trust-region'' of $\mu$~\cite{schulman2015trust,achiam2017constrained,levin2017markov}, so that we can use $s \sim d^{\mu}$ in Equation (\ref{eq_performance_diff}) to get a surrogate objective that approximates $\Delta(\pi, \mu)$. 
\begin{align}
\label{eq_surrogate}
\Delta(\pi, \mu) \geq \frac{1}{1-\gamma} \underset{\substack{s \sim d^{\mu} \\ a \sim \pi}}{\mathbb{E}}\left[A^{\mu}(s,a)-\frac{2 \gamma \epsilon^{\mu} }{1-\gamma} D_{TV}(\pi(\cdot | s) \| \mu(\cdot | s))\right],
\end{align}
where $\epsilon^{\mu}=\max _{s \in \mathcal{S}}\left|\mathbb{E}_{a \sim \pi(\cdot | s)}\left[A^{\mu}(s, a)\right]\right|$ and $D_{TV}$ is the total variation distance.
There are also other ways to make two policies stay close, e.g., constraining action probabilities by mixing policies~\cite{kakade2002approximately,pirotta2013safe} or clipping the surrogate objective \cite{schulman2017proximal,wang2020truly}.

However, as discussed in section~\ref{sec_offline_rl}, constraining the learned policy to stay close to the behavior policy will be too restrictive\footnote{Note that in the online setting, this issue is not problematic as we replace $\mu$ with $\pi_i$ at each iteration $i$, thus ensuring the learned policy with monotonically increased performance.}. 
Furthermore, constraining the distributions of the immediate future actions might not be enough to ensure that the resulting surrogate objective (\ref{eq_surrogate}) is still a valid estimate of the performance of the next policy. This will result in instability and premature convergence, especially in long-horizon problems.
Instead, we directly apply importance sampling~\citep{degris2012off} on Equation (\ref{eq_performance_diff}), yielding
\begin{align}
\label{eq_is}
\Delta(\pi, \mu) = \frac{1}{1-\gamma} \ \underset{\substack{s \sim d^{\mu} \\ a \sim \pi}}{\mathbb{E}} \left[ \frac{d^{\pi}(s)}{d^{\mu}(s)} \ A^{\mu}(s,a) \right].
\end{align}
this objective reasons about the long-term effect of the policies on the distribution of future states, and can be optimized fully offline given that we can compute the state visitation ratio $d^{\pi}(s)/d^{\mu}(s)$.
Let $w^{\pi}(s) = d^{\pi}(s)/d^{\mu}(s)$, now the remaining question is how to estimate $w^{\pi}(s)$ using offline data.

We aim to estimate it by using the steady-state property of Markov processes
\cite{liu2018breaking,gelada2019off},
given by the following theorem and the proof is provided in Appendix A.
\begin{theorem} Assume $\mu(a|s) > 0$ whenever $\pi(a|s) > 0$, we have function $w(s) = w^{\pi}(s)$ if and only if it satisfies
\begin{align*}
&\underset{\left(s, a, s^{\prime}\right) \sim d^{\mu}}{\mathbb{E}}\left[ \mathbb{D}(w(s')~\|~\mathcal{T}^{\pi}w(s')) \right]=0, \quad \forall s'. \\
&\text { with } \quad \mathcal{T}^{\pi} w(s') := (1-\gamma) + \gamma \underset{(s,a)|s'}{\mathbb{E}} \frac{\pi(a|s)}{\mu(a|s)} w(s),
\end{align*}
where $\mathbb{D}(\cdot \| \cdot)$ is some discrepancy function between distributions and $(s,a)|s'$ is a time-reserved conditional probability, it is the conditional distribution of $(s,a)$ given that their next state is $s'$ following policy $\mu$.
\end{theorem}

Note that the operator $\mathcal{T}^{\pi}$ is different from the standard Bellman operator,
although they bear some similarities. More specifically, given some state-action pair $(s, a)$, the Bellman operator is defined using next state $s'$ of $(s, a)$, while $\mathcal{T}^{\pi}$ is defined using previous state-actions $(s, a)$ that transition to $s'$. In this sense, $\mathcal{T}^{\pi}$ is time-backward. The function $w$ actually has the interpretation of the distribution over $\mathcal{S}$. Therefore, $\mathcal{T}^{\pi}$ describes how visitation flows from previous state $s$ to next state $s'$, which is called the \textit{backward flow operator} \cite{mousavi2019black}. 
Actually, the state visitation ratio $w^\pi$ is the unique fixed point of $\mathcal{T}^{\pi}$, that is $w^{\pi}=\mathcal{T}^{\pi}w^{\pi}$, we use this important property to formulate Theorem 1.
We also note that similar forms of $\mathcal{T}^{\pi}$ have appeared in the safe RL literature, usually used to encode constraints in a dual linear program for an MDP (e.g., \citealp{wang2008stable,satija2020constrained}).
However, the application of $\mathcal{T}^{\pi}$ for the state visitation ratio estimation problem and apply it in the offline RL problem appears new to the best of our knowledge.

There are many choices for $\mathbb{D}(\cdot \| \cdot)$ and different solution approaches (e.g., \citealp{nguyen2010estimating,dai2017learning}). 
In this work, we use the approach based on kernel Maximum Mean Discrepancy (MMD) \cite{muandet2017kernel}.
Given real-valued function $f$ and $g$ that defined on $\mathcal{S}$, we define the following bilinear functional $\mathbf{k}[\cdot;\cdot]$ as
\begin{align}
\label{eq_functional}
\mathbf{k}[f;g] :=\sum_{s \in \mathcal{S}, \bar{s} \in \mathcal{S}} f(s) k(s, \bar{s}) g(\bar{s}),
\end{align}
where $k(\cdot, \cdot)$ is a positive definite kernel function defined on $\mathcal{S}\times\mathcal{S}$, such as Laplacian and Gaussian kernels.

Let $\mathcal H$ be the reproducing kernel Hilbert space (RKHS) associated with the kernel function $\mathbf{k}$, the MMD between two distributions, $\mu_1$ and $\mu_2$, is defined by
\begin{equation*}
\mathbb{D}_{M}\left(\mu_{1} \| \mu_{2}\right):=\sup _{f \in \mathcal{H}}\left\{\mathbb{E}_{\mu_{1}}[f]-\mathbb{E}_{\mu_{2}}[f], \ \text { s.t. } \ \|f\|_{\mathcal{H}} \leq 1\right\}.
\end{equation*}
Here, $f$ may be considered as a discriminator, playing a similar role as the discriminator network in generative adversarial networks \cite{goodfellow2014generative}, to measure the difference between $\mu_1$ and $\mu_2$. 
A useful property of MMD is that it admits a closed-form expression~\cite{gretton2012kernel}:
\begin{align*}
\mathbb{D}_{M}\left(\mu_{1} \| \mu_{2}\right) &=\mathbf{k}\left[\mu_{1}-\mu_{2} ; \mu_{1}-\mu_{2}\right] \\
&=\mathbf{k}\left[\mu_{1} ; \mu_{1}\right]-2 \mathbf{k}\left[\mu_{1} ; \mu_{2}\right]+\mathbf{k}\left[\mu_{2} ; \mu_{2}\right],
\end{align*}
where $\mathbf{k}[\cdot;\cdot]$ is defined in Equation (\ref{eq_functional}), and we use its bilinear property. 
The expression for MMD does not involve the density of either distribution $\mu_1$  or $\mu_2$, and it can be optimized directly through sample-based calculation.

Given independent transition samples $\{(s_i, a_i, s'_i),$ $(s_j, a_j, s'_j)\} \sim \mathcal{B}$, applying MMD to Theorem 1 produces 
\begin{align}
\label{eq_mmd}
&\mathbb{D}_{M}\left(w \| \mathcal{T}^{\pi} w\right)=\mathbf{k}\left[w ; w\right]-2 \mathbf{k}\left[w ; \mathcal{T}^{\pi}w\right]+\mathbf{k}\left[\mathcal{T}^{\pi}w;\mathcal{T}^{\pi} w\right] \notag \\ 
&= \sum_{i,j} w(s'_i) w(s'_j) \cdot k(s'_i,s'_j) -2 \sum_{i,j} w(s'_i) \Big[1-\gamma + \gamma \frac{\pi(a_j|s_j)}{\mu(a_j|s_j)} w(s_j)\Big]  \notag \\
{}&\cdot k(s'_i,s_j) + \sum_{i,j} \Big[1-\gamma + \gamma \frac{\pi(a_i|s_i)}{\mu(a_i|s_i)} w(s_i)\Big] \Big[1-\gamma + \gamma \frac{\pi(a_j|s_j)}{\mu(a_j|s_j)} w(s_j)\Big] \notag \\
{}&\cdot k(s_i,s_j).
\end{align}

In above formulation, both $w$ and $\mathcal{T}^{\pi}w$ are probability mass functions on $\mathcal{S}$, consisting of state-actions encountered in offline data $\mathcal{B}$. Therefore, we can optimize this objective by mini-batch training.
  
Recall that Theorem 1 needs to satisfy $\pi(a|s) > 0 \Longrightarrow \mu(a|s)>0$ to prevent $\mu(a|s)=0$. 
This condition actually means that $\pi$ should lie in the support of $\mu$, we accomplish this by using a log-barrier as a support measure, as the log function decreases exponentially when the probability densities of actions sampled from the learned policy are small under the behavior policy (i.e., $\mathbb{E}_{a \sim \pi(\cdot|s)} \left[ \log \mu(a|s) \right]$ is small). In practice, the knowledge of $\mu$ is not explicitly provided, and one usually uses behavioral cloning on the offline dataset to approximate $\mu$~\cite{wu2019behavior}. 
Note that since $\mu$ is trained as a normalized probability distribution, it will assign low probabilities to actions outside of $\mathcal{B}$.

After introducing our method to estimate state visitation ratio $w^{\pi}$, we now go back to Equation (\ref{eq_is}). Recall that $A^{\mu}(s, a) = Q^{\mu}(s, a) - V^{\mu}(s)$ and $V^{\mu}(s)$ is a constant with respect to $\pi$. Compiling all the above results, we can get the learning objective of $\pi$ as
\begin{align*}
\max_{\pi} \underset{s \sim \mathcal{B}}{\mathbb{E}} \left[ \underset{a \sim \pi(\cdot|s)}{\mathbb{E}} \left[ w^{\pi}(s) Q^{\mu}(s,a) \right] \right] \\
\mathrm{s.t.} \underset{a \sim \pi(\cdot|s)}{\mathbb{E}} \left[ \log \mu(a|s) \right] \geq \epsilon.
\end{align*}

We convert the constrained optimization problem to an unconstrained one by 
treating the constraint term as a penalty term~\cite{bertsekas1997nonlinear,le2019batch,wu2019behavior}, and we finally get the learning objective  as 
\begin{align}
\label{eq_objective}
\max_{\pi} \underset{s \sim \mathcal{B}}{\mathbb{E}} \left[ \underset{a \sim \pi(\cdot|s)}{\mathbb{E}} \left[ w^{\pi}(s) Q^{\mu}(s,a) + \alpha \log \mu(a|s) \right] \right].
\end{align}
Notice that $\mu(\cdot|s) \leq 1$, so we are optimizing a lower bound of the performance difference $\Delta(\pi, \mu)$, and this lower bound is much tighter than (\ref{eq_surrogate}) in most cases. 
Such policy improvement guarantee is not enforced in other offline RL algorithms.

\subsection{Soft Regularization Derivation}
\label{sec_regularization_derivation}
We now provide a different view of our learning objective (\ref{eq_objective}), we can rearrange its form by changing the weight $w^{\pi}(s)$ from $Q^{\mu}$ to the log-barrier $\log \mu(a|s)$, yielding
\begin{align*}
\max_{\pi}\underset{s \sim \mathcal{B}}{\mathbb{E}} \left[ \underset{a \sim \pi(\cdot|s)}{\mathbb{E}} \left[ Q^{\mu}(s,a) \right] - \alpha \ \frac{d^{\mu}(s)}{d^{\pi}(s)} \underset{a \sim \pi(\cdot|s)}{\mathbb{E}} \left[ -\log \mu(a|s) \right] \right].
\end{align*}

Comparing to (\ref{eq_behavior_reg}), the objective used in previous offline RL algorithms, we actually have the following three changes: 
\textbf{(1)} We no longer need to estimate $Q^{\pi}$, instead, we only need to estimate $Q^{\mu}$, $Q^{\mu}$ can be much easily estimated and robust than $Q^{\pi}$, as the behavior policy $\mu$ remains fixed during training. It will no longer suffer from the over-estimation issue when computing the target Q values.
\textbf{(2)} We use KL divergence that excludes the learned policy's entropy as the regularization term, as $\mathbb{E}_{a \sim \pi(\cdot|s)} \left[ -\log \mu(a|s) \right] = D_{KL}(\pi(\cdot|s),\mu(\cdot|s)) - \mathbb{E}_{a \sim \pi(\cdot|s)} \left[ \log \pi(a|s) \right]$.
The entropy term $\mathbb{E}_{a \sim \pi(\cdot|s)} \left[ \log \pi(a|s) \right]$ will result in more stochastic policy distribution, we argue that this property is only useful for exploration in the online setting~\cite{haarnoja2018soft} and will do harm to the offline setting as the optimal policy-induced from the offline data is close to deterministic.
\textbf{(3)} We use a state-dependent regularization weight, instead of the state-independent regularization weight used in prior approaches.
When $d^{\pi}(s)$ is small, the regularization term will be enlarged to make the policy match the state visitation distribution under the behavior policy, $d^{\mu}(s)$. This will alleviate the distribution shift problem~\cite{levine2020offline,fujimoto2021minimalist}, making the learned policy more stable and robust when evaluating in the environment. When $d^{\pi}(s)$ is large, it means that the policy $\pi$ has a higher chance to visit the state $s$, the regularization term will thus become small, allowing more freedom of behavior policy deviation,
thus leading to better performance.

These three changes and our policy improvement guarantee perfectly address the three aforementioned challenges. We call our algorithm \textbf{S}oft \textbf{B}ehavior-regularized \textbf{A}ctor \textbf{C}ritic (\textbf{SBAC})
, we present the pseudocode of our method in Algorithms \ref{alg1} and implementation details in Appendix B.

\begin{algorithm}
\caption{Soft Behavior-regularized Actor Critic (SBAC)}
\label{alg1}
\begin{algorithmic}[1]
\REQUIRE 
Dataset $\mathcal{B}$, weight $\alpha$.
\STATE Initialize behavior model $\mu$, model $w$, policy network $\pi_\theta$, critic network $Q_{\psi}$ and target critic network $Q_{\psi^{\prime}}$ with $\psi^{\prime} \leftarrow \psi$.
\FOR{$t=0,1,...,M$}
\STATE Sample mini-batch transitions $(s,a, r, s', a') \sim \mathcal{B}$.
\STATE Update $\mu$ by maximizing $\log \mu(a|s)$.
\STATE Set $y=Q_{\psi'}\left(s', a'\right)$.
\STATE Update $Q_{\psi}$ by minimizing $\left(r+\gamma y - Q_{\psi}(s, a) \right)^{2}$.
\STATE Update $Q_{\psi^{\prime}}$ by $\psi' \leftarrow \tau \psi+(1-\tau)\psi'$.
\ENDFOR
\FOR{$t=0,1,...,N$}
\STATE Sample mini-batch of transitions $(s,a,r,s') \sim \mathcal{B}$.
\STATE Update $w$ by minimizing $\mathbb{D}_{M}\left(w \| \mathcal{T}^{\pi} w\right)$, which can be computed by Equation (\ref{eq_mmd}).
\STATE Update $\pi_\theta$ by (\ref{eq_objective}).
\ENDFOR
\end{algorithmic}
\end{algorithm}

\section{Experiments}
We construct experiments on both widely-used D4RL MuJoCo datasets and more complex Adroit hand manipulation environment (visualized in Figure \ref{fig_env}).

We compare SBAC with several strong baselines, including policy regularization methods such as BCQ~\cite{fujimoto2019off}, BEAR~\cite{kumar2019stabilizing}, and BRAC-p/v~\cite{wu2019behavior}; critic penalty methods such as CQL~\cite{kumar2020conservative} and F-BRC~\cite{kostrikov2021offline}. We also compare our method with BRAC+~\cite{zhang2020brac+}, which employed state-dependent behavior regularization, and AlgaeDICE~\cite{nachum2019algaedice}, which constrained the state distribution shift by applying state-visitation-ratio regularization.

\begin{figure*} [htb]
\centering
\includegraphics[width=2.1\columnwidth]{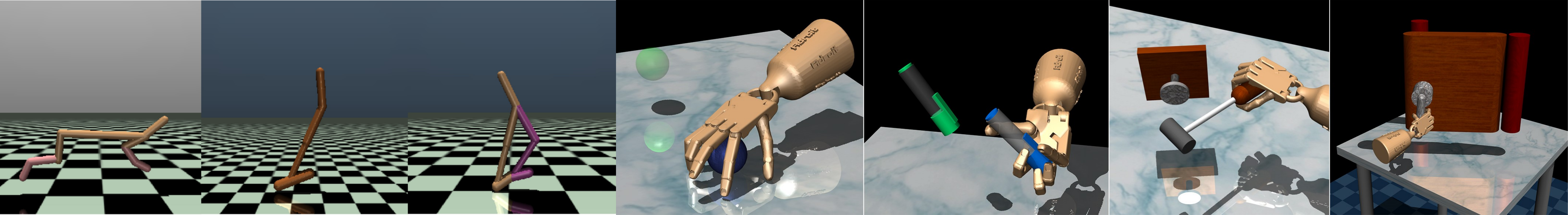}
\caption{Visualization of used environments. From left to right: HalfCheetah, Hopper, Walker2d, Relocate, Pen, Hammer, Door.}
\label{fig_env}
\end{figure*}

\begin{table*}[htb]
\centering
\caption{Results for Mujoco and Adroit datasets. Each value is the normalized score proposed in \protect\cite{fu2020d4rl} of the policy at the last iteration of training, averaged over 5 random seeds. The scores are undiscounted average returns normalized to lie between 0 and 100, where a score of 0 corresponds to a random policy and 100 corresponds to an expert. Results of BCQ, BEAR, BRAC-p, BRAC-v, AlgaeDICE and CQL are taken from \protect\cite{fu2020d4rl,wu2021uncertainty,zhang2020brac+}. Results of BRAC+ and F-BRAC are taken from their original paper. The highest values are marked bold.
}
\begin{tabular}{ccccccccc}
\hline
\multicolumn{1}{c|}{\multirow{2}{*}{Task name}} & \multicolumn{5}{c|}{Policy Regularization}                                                                     & \multicolumn{3}{c}{Critic Penalty}                                 \\ \cline{2-9} 
\multicolumn{1}{c|}{}                           & BCQ                  & BEAR                 & BRAC-p/v             & BRAC+   & \multicolumn{1}{c|}{SBAC(Ours)} & AlgaeDICE            & CQL                  & F-BRC                \\ \hline
\multicolumn{1}{c|}{cheetah-random}             & 2.2                  & 25.1                 & 24.1/31.2            & 26.4±1.0 & \multicolumn{1}{c|}{28.7±1.2}    & -0.3                 & \textbf{35.4}              & 33.3±1.3             \\
\multicolumn{1}{c|}{walker-random}              & 4.9                  & 7.3                  & -0.2/1.9             & \textbf{16.7}±2.3 & \multicolumn{1}{c|}{1.3±1.7}    & 0.5                  & 7.0              & 1.5±0.7              \\
\multicolumn{1}{c|}{hopper-random}              & 10.6                 & 11.4                 & 11.0/12.2            & 12.5±0.3 & \multicolumn{1}{c|}{\textbf{16.7}±1.3}    & 0.9                  & 10.8              & 11.3±0.2             \\
\multicolumn{1}{c|}{cheetah-medium}             & 40.7                 & 41.7                 & 43.8/46.3            & 46.6±0.6 & \multicolumn{1}{c|}{\textbf{51.4}±0.3}    & -2.2                 & 44.4              & 41.3±0.3             \\
\multicolumn{1}{c|}{walker-medium}              & 53.1                 & 59.1                 & 77.5/81.1            & 75.1±3.5 & \multicolumn{1}{c|}{\textbf{81.1}±1.1}    & 0.3                  & 79.2              & 78.8±1.0             \\
\multicolumn{1}{c|}{hopper-medium}              & 54.5                 & 52.1                 & 32.7/31.1            & 53.2±3.1 & \multicolumn{1}{c|}{\textbf{99.6}±1.7}    & 1.2                  & 58.0              & \textbf{99.4}±0.3             \\
\multicolumn{1}{c|}{cheetah-medium-replay}      & 38.2                 & 38.6                 & 45.4/47.7            & 46.1±0.2 & \multicolumn{1}{c|}{\textbf{48.2}±0.2}    & -2.1                 & 46.2              & 43.2±1.5             \\
\multicolumn{1}{c|}{walker-medium-replay}       & 15.0                 & 19.2                 & -0.3/0.9             & 39.0±4.6 & \multicolumn{1}{c|}{\textbf{78.3}±1.2}    & 0.6                  & 26.7              & 41.8±7.9             \\
\multicolumn{1}{c|}{hopper-medium-replay}       & 33.1                 & 33.7                 & 0.6/0.6              & 72.7±18.9 & \multicolumn{1}{c|}{\textbf{98.9}±0.8}   & 1.1                  & 48.6              & 35.6±1.0             \\
\multicolumn{1}{c|}{cheetah-expert}             & /                    & \textbf{108.2}               & 3.8/-1.1             & / & \multicolumn{1}{c|}{98.7±0.1}     & /                    & 104.8              & \textbf{108.4}±0.5            \\
\multicolumn{1}{c|}{walker-expert}              & /                    & 106.10               & -0.2/0.0             & / & \multicolumn{1}{c|}{113.7±0.2}     & /                    & \textbf{153.9}              & 103.0±5.3            \\
\multicolumn{1}{c|}{hopper-expert}              & /                    & 110.30               & 6.6/3.7              & / & \multicolumn{1}{c|}{\textbf{112.2}±0.1}     & /                    & 109.9              & \textbf{112.3}±0.1            \\
\multicolumn{1}{c|}{cheetah-medium-expert}      & 64.7                 & 53.4                 & 44.2/41.9            & 61.2±2.8 & \multicolumn{1}{c|}{\textbf{93.1}±0.8}    & -0.8                 & 62.4              & \textbf{93.3}±10.2            \\
\multicolumn{1}{c|}{walker-medium-expert}       & 57.5                 & 40.1                 & 76.9/81.6            & 95.3±5.9 & \multicolumn{1}{c|}{\textbf{112.4}±1.5}    & 0.4                  & 98.7              & 105.2±3.9            \\
\multicolumn{1}{c|}{hopper-medium-expert}       & 110.9                & 96.3                 & 1.9/0.8              & \textbf{112.9}±0.1 & \multicolumn{1}{c|}{\textbf{112.8}±2.3}    & 1.1                  & 111.0              & \textbf{112.4}±0.3            \\ \hline
\multicolumn{1}{c|}{Mean performance}           & 40.5                  & 53.5                  & 24.5/25.3              & 54.8±3.6 & \multicolumn{1}{c|}{\textbf{76.4}±1.0}    & 0.1                  & 66.5              & 68.0±2.3              \\ \hline
\multicolumn{1}{c|}{pen-human}                  & 68.9                 & -1.0                  & 8.1/0.6              & 64.9±1.6 & \multicolumn{1}{c|}{\textbf{78.8}±2.0}    & -3.3                  & 55.8              & /                    \\
\multicolumn{1}{c|}{hammer-human}               & 0.5                  & 0.3                  & 0.3/0.2              & 3.9±0.9 & \multicolumn{1}{c|}{\textbf{7.7}±1.1}    & 0.3                  & 2.1              & /                    \\
\multicolumn{1}{c|}{door-human}                 & 0.0                  & -0.3                  & -0.3/-0.3              & \textbf{11.5}±1.2 & \multicolumn{1}{c|}{10.1±3.5}    & -0.0                  & 9.1              & /                    \\
\multicolumn{1}{c|}{relocate-human}             & -0.1                 & -0.3                  & -0.3/-0.3              & 0.2±0.11 & \multicolumn{1}{c|}{\textbf{0.6}±0.3}    & -0.1                  & 0.4              & /                    \\ \hline
\multicolumn{1}{c|}{Mean performance}           & 17.3                  & -0.3                  & 1.0/0.1              & 20.1±1.0 & \multicolumn{1}{c|}{\textbf{24.3}±1.7}    & -0.8                  & 16.9              & /                    \\ \hline
\multicolumn{1}{l}{}                            & \multicolumn{1}{l}{} & \multicolumn{1}{l}{} & \multicolumn{1}{l}{} &         & \multicolumn{1}{l}{}            & \multicolumn{1}{l}{} & \multicolumn{1}{l}{} & \multicolumn{1}{l}{}
\end{tabular}
\label{table_d4rl}
\end{table*}

\subsection{Performance on benchmarking datasets for offline RL}
We evaluate our method on the MuJoCo datasets in the D4RL benchmarks~\cite{fu2020d4rl}, including three environments (halfcheetah, hopper, and walker2d) and five dataset types (random, medium, medium-replay, medium-expert, expert), yielding a total of 15 problem settings. The datasets in this benchmark have been generated as follows: \textbf{random}: roll out a randomly initialized policy for 1M steps. \textbf{expert}: 1M samples from a policy trained to completion with SAC. \textbf{medium}: 1M samples from a policy trained to approximately 1/3 the performance of the expert. \textbf{medium-replay}: replay buffer of a policy trained up to the performance of the medium agent. \textbf{medium-expert}: 50-50 split of medium and expert data.

The results of SBAC and competing baselines on MoJoCo datasets are reported in Table \ref{table_d4rl}. SBAC consistently outperforms all policy regularization baselines (BCQ, BEAR, BRAC-p/v, BRAC+) in almost all tasks, sometimes by a large margin. This is not surprising, as most policy regularization baselines use restrictive, state-independent regularization penalties, which lead to over-conservative policy learning. Compared to the less restrictive critic penalty methods, SBAC also surpasses the performance of strong baselines like CQL and F-BRC in a large fraction of tasks. For a few tasks that F-BRC performs better, SBAC achieves comparable performance. It is also observed that SBAC indeed learns more robust and low-variance policies, which have a small standard deviation over seeds compared with policy regularization (BRAC+) and critic penalty (F-BRC) baselines.

\subsection{Performance on Adroit hand tasks with human demonstrations}
We then experiment with a more complex robotic hand manipulation dataset. The Adroit dataset~\cite{rajeswaranlearning} involves controlling a 24-DoF simulated hand to perform 4 tasks including hammering a nail, opening a door, twirling a pen, and picking/moving a ball. This dataset is particularly hard for previous state-of-the-art works in that it contains narrow human demonstrations on a high-dimensional robotic manipulation task. 

It is found that SBAC dominates almost all baselines and achieves state-of-the-art performance. The only exception is the door-human task, in which SBAC reaches comparable performance with BRAC+. This demonstrates the ability of SBAC to learn from complex and non-markovian human datasets.

\subsection{Analysis of the effect of distribution correction}
A notable issue with existing offline RL algorithms is that they exhibit huge variance in performance during evaluation as compared to online trained policies, likely caused by distributional shift and poor generalization \cite{fujimoto2021minimalist}. In Figure \ref{fig_min_eval} and \ref{fig_min_10_eval}, we report the instability test results of SBAC as well as baseline methods CQL, FisherBRC, and an online TD3~\cite{fujimoto2018addressing} policy across 10 evaluations at a specific point or over a period of time.
SBAC achieves surprisingly low variances in performance that beats all other offline RL algorithms and exhibits a strong ability of distribution error correction. This is extremely important for many safety-aware real world tasks, which require robust and predictable policy outputs.

As discussed in section~\ref{sec_regularization_derivation}, this mainly results from the following three reasons. First, the value function $Q^\mu$ is completely learned from the data via Fitted Q-evaluation~\cite{le2019batch}, which is much easier and more robust to estimate. Since SBAC does not rely on a Q-value function $Q^\pi$ bond to the learned policy as in typical RL algorithms, it no longer suffers from the accumulating exploitation error in the target values during the bootstrap updates of Q-values.
Second, the removal of the entropy term in the regularization penalty helps to reduce policy entropy, leading to a more uniform and deterministic behavior.
Lastly, the introduction of state-dependent regularization in SBAC places different weights depending on the state visitation ratio, which encourages the learned policy to match the state visitation distribution of the behavior policy.

\begin{figure} [htb]
\centering
\includegraphics[width=1.0\columnwidth]{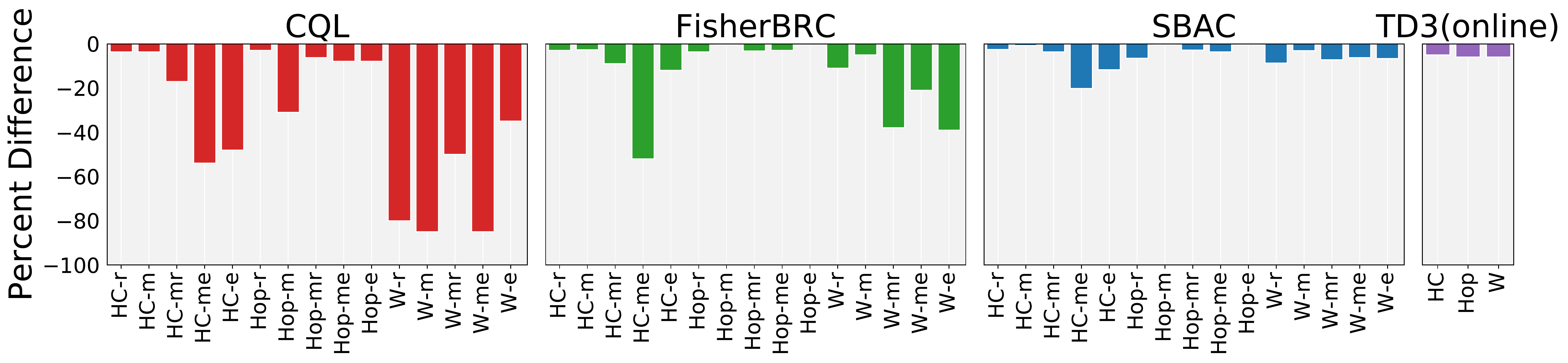}
\caption{Percent difference of the worst episode during the 10 evaluation episodes at the last evaluation. This measures the deviations in performance at a single point in time. HC = HalfCheetah, Hop = Hopper, W = Walker, r = random, m = medium, mr = medium-replay, me = medium-expert, e = expert. It can be shown that online algorithms (TD3) typically have small episode variances as they won't suffer from distribution shift, while our method achieves smaller episode variances when comparing to other offline RL algorithms.}
\label{fig_min_eval}

\includegraphics[width=1.0\columnwidth]{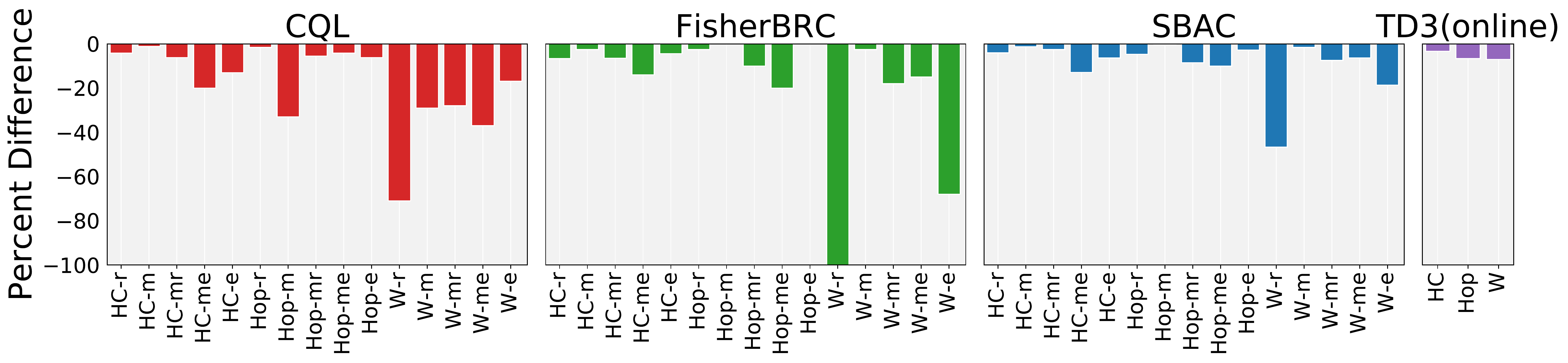}
\caption{Percent difference of the worst evaluation during the last 10 evaluations. This measures the deviations in performance over a period of time. Definitions of x-axis labels are the same as in Figure \ref{fig_min_eval}.
Similar to the result in Figure \ref{fig_min_eval}, all offline RL methods have significant variances near the final stage of training that are absent in the online setting, and our method can alleviate this phenomenon.}
\label{fig_min_10_eval}
\end{figure}

\subsection{Ablations}
In section~\ref{sec_regularization_derivation}, we discussed that compared to objective (\ref{eq_behavior_reg}), SBAC replace $Q^{\pi}$ with $Q^{\mu}$ and use the state visitation ratio $w^{\pi}(s)$ as the state-dependent regularization weight. We now aim to dive deeper into each component and investigate its impact on the training performance.

\noindent \textbf{Ablation 1} Our first ablation isolates the effect
of $w^{\pi}(s)$ on the performance. Formally, we use the following objective to learn the policy, with everything else remains unchanged.
\begin{align*}
\max_{\pi}\underset{s \sim \mathcal{B}}{\mathbb{E}} \left[ \underset{a \sim \pi(\cdot|s)}{\mathbb{E}} \left[ Q^{\mu}(s,a) \right] - \alpha \ \underset{a \sim \pi(\cdot|s)}{\mathbb{E}} \left[ -\log \mu(a|s) \right] \right].
\end{align*}

This objective is similar to objective (\ref{eq_surrogate}), the surrogate objective of online RL methods introduced in section~\ref{sec_improvement_derivation}. In Figure \ref{fig_ablation}, it can be seen that using this objective is actually doing one-step policy optimization against the behavior policy, which is expected to be suboptimal~\cite{brandfonbrener2021offline}.

\noindent \textbf{Ablation 2} Our second ablation isolates the effect
of $Q^{\mu}$ on the performance, formally, we use the following objective to learn the policy
\begin{align*}
\max_{\pi}\underset{s \sim \mathcal{B}}{\mathbb{E}} \left[ \underset{a \sim \pi(\cdot|s)}{\mathbb{E}} \left[ Q^{\pi}(s,a) \right] - \alpha \ \underset{a \sim \pi(\cdot|s)}{\mathbb{E}} \left[ -\log \mu(a|s) \right] \right].
\end{align*}

This objective is almost the same with objective (\ref{eq_behavior_reg}) that uses KL divergence but removes the learned policy's entropy, it can be shown in Figure \ref{fig_ablation} that the learning procedure is more unstable due to the inaccurate estimation of $Q^{\pi}$. Furthermore, the performance of this ablation is lower than SBAC, suggesting that this objective is too restrictive.

\begin{figure*} [htb]
\centering
\includegraphics[width=1.0\columnwidth]{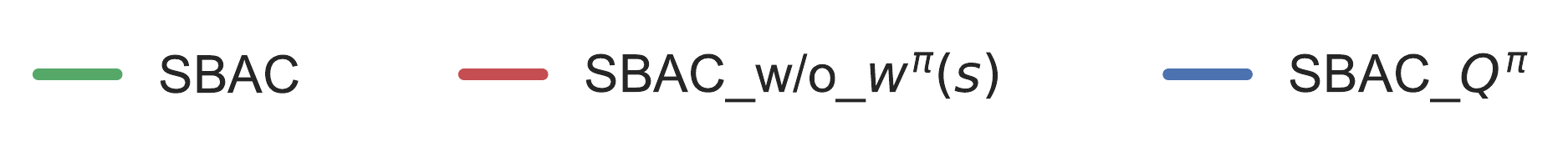}
\includegraphics[width=2.1\columnwidth]{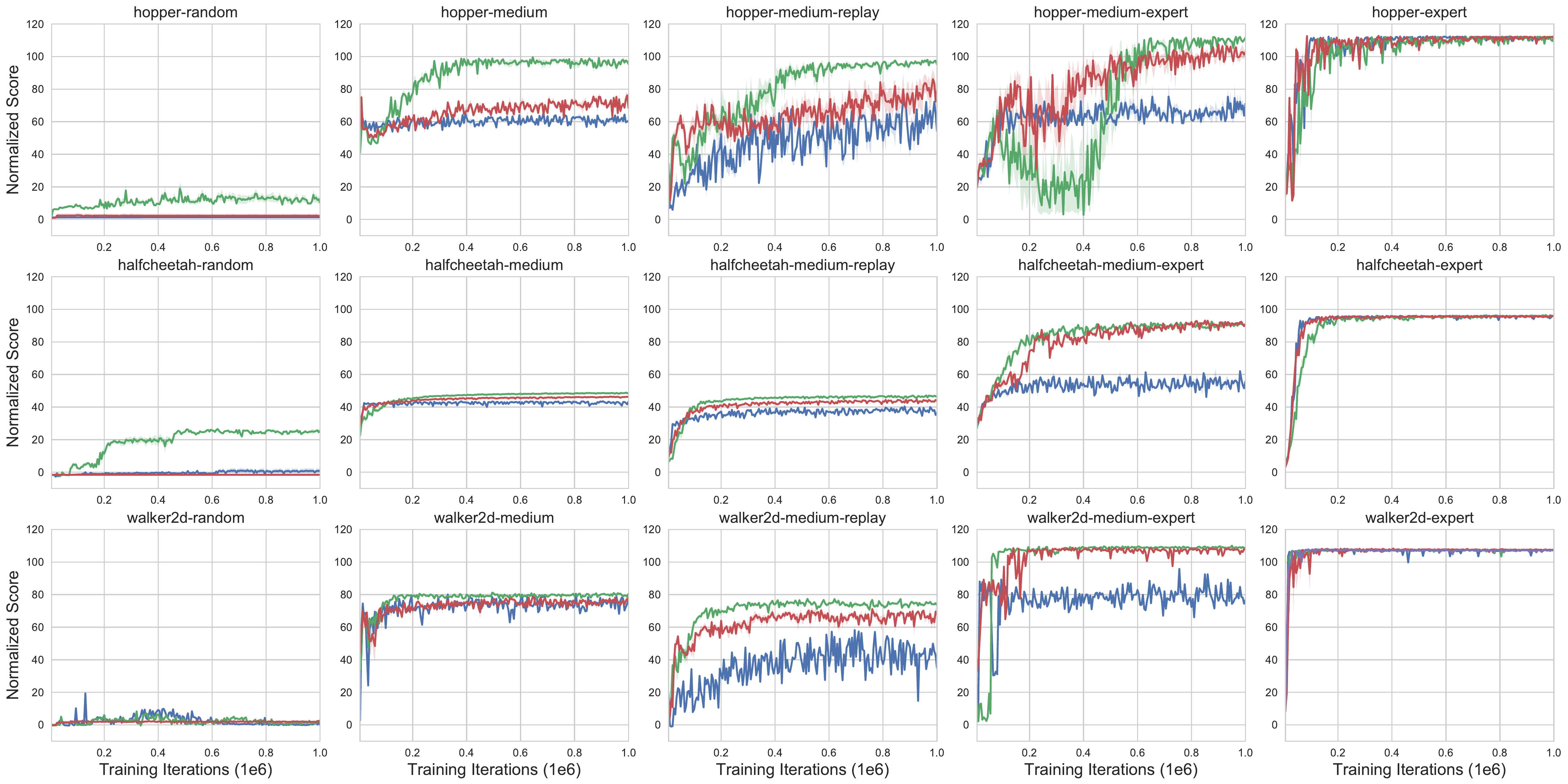}
\caption{Learning curves comparing the performance of SBAC against its ablations in the D4RL datasets. Curves are averaged over 5 seeds, with the shaded area representing the standard deviation across seeds. SBAC achieves higher results comparing to the ablation without $w^{\pi}(s)$ and lower std comparing to the ablation without $w^{\pi}(s)$ and with $Q^{\pi}$ instead of $Q^{\mu}$.}
\label{fig_ablation}
\end{figure*}

\section{Related Work}
Our work contributes to the literature in behavior regularization methods of offline RL~\cite{wu2019behavior}, the primary ingredient of this class of methods is to propose various policy regularizers to ensure that the learned policy does not stray too far from the data generated by the behavior policy.
These regularizers can appear explicitly as divergence penalties~\cite{wu2019behavior,kumar2019stabilizing,fujimoto2021minimalist}, implicitly through weighted behavior cloning~\cite{wang2020critic,peng2019advantage,nair2020accelerating}, or more directly through careful parameterization of the policy~\cite{fujimoto2018addressing,zhou2020latent}. 
Another way to apply behavior regularizers is via modification of the critic learning objective to incorporate some form of regularization to encourage staying near the behavioral distribution and being pessimistic about unknown state-action pairs~\cite{nachum2019algaedice,kumar2020conservative,kostrikov2021offline}. However, being able to properly quantify the uncertainty of unknown states is difficult when dealing with neural network value functions~\cite{buckman2020importance,xu2021constraints}.

In this work, we demonstrate the benefit of using a state-dependent behavior regularization term. This draws the connection to other methods~\cite{laroche2019safe,sohn2020brpo}, that bootstrap the learned policy with the behavior policy only when the current state-action pair's uncertainty is high, allowing the learned policy to differ from the behavior policy for the largest improvements. However, these methods measure the uncertainty by the visitation frequency of state-action pairs in the dataset, which is computationally expensive and nontrivial to apply in continuous control settings. There are also some methods trying to measure the state-dependent regularization weights by neural networks~\cite{zhang2020brac+,lee2020batch}, however, these weights are hard to train and lacks physical interpretation, while also performing inferior compared with our method.

In our learning objective, we need to estimate the discounted marginal state visitation ratio $d^{\pi}(s)/d^{\mu}(s)$. This draws the connection to the Off-policy Evaluation (OPE) literature, which aims to estimate the performance of a policy given data from a different policy~\cite{hallak2017consistent,liu2018breaking,nachum2019dualdice,gelada2019off,mousavi2019black}. There are also some work incorporating marginal state visitation difference between the learned policy and offline data into the learning objective~\cite{nachum2019algaedice,lee2021optidice}, however, optimizing these resulting objectives need to solve difficult min-max optimization problems, which is susceptible to instability and leads to poor performance. 
This can be shown in our experiments that AlgaeDICE performs poorly and barely solves any tasks. Note that our work is unique in using the marginal state visitation ratio as the regularization weights.

\section{Conclusions and Limitations}
In this paper, we propose a simple yet effective offline RL algorithm. Existing model-free behavior regularized offline RL methods are overly restrictive and often have unsatisfactory performance. Motivated by this limitation, we design a new behavior regularization scheme for offline RL that enables policy improvement guarantee and state-dependent policy regularization. The resulting algorithm, SBAC, performs strongly against state-of-the-art methods in both the standard D4RL dataset and the more complex Adroit tasks. 
One limitation of SBAC is the need to estimate a behavior policy, which may be impacted by the expressiveness of the behavior policy. We leave it as future work to try to avoid estimating it.

\bibliography{main}

\appendix
\onecolumn

\section{Proof}
\noindent \textbf{Notations}
Given a learned policy $\pi$,
let $d_{t}^{\pi}(s)$ be the marginal distribution of $s_{t}$ under $\pi$, that is, $d_{t}^{\pi}(s):=\operatorname{Pr}\left[s_{t}=s | s_{0} \sim \rho, \pi\right], d_{t}^{\pi}(s, a)=d_{t}^{\pi}(s) \pi(a | s)$, and the state-action \textit{policy} generating distribution $d^\pi$ can be expressed as $d^{\pi}(s, a)=(1-\gamma) \sum_{t=0}^{\infty} \gamma^{t} d_{t}^{\pi}(s, a)$. 
Similarly, denote the state-action \textit{data} generating distribution as $d^{\mu}$, induced by some data-generating (behavior) policy $\mu$, that is, $\left(s_{i}, a_{i}\right) \sim d^{\mu}(s,a)$ for $\left(s_{i}, a_{i}, s_{i}^{\prime}, r_{i} \right) \in \mathcal{B}$. Note that data set $\mathcal{B}$ is formed by multiple trajectories generated by $\mu$. 
For each $\left(s_{i}, a_{i}\right)$, we have $s_{i}^{\prime} \sim P\left(\cdot | s_{i}, a_{i}\right)$, $r_{i}=r\left(s_{i}, a_{i}\right)$.

\begin{theorem} Assume $\mu(a|s) > 0$ whenever $\pi(a|s) > 0$, we have function $w(s) = w^{\pi}(s)$ if and only if it satisfies
\begin{align*}
&\underset{\left(s, a, s^{\prime}\right) \sim d^{\mu}}{\mathbb{E}}\left[ \mathbb{D}(w(s')~\|~\mathcal{T}^{\pi}w(s')) \right]=0, \quad \forall s'. \\
&\text { with } \quad \mathcal{T}^{\pi} w(s') := (1-\gamma) + \gamma \underset{(s,a)|s'}{\mathbb{E}} \frac{\pi(a|s)}{\mu(a|s)} w(s),
\end{align*}
where $\mathbb{D}(\cdot \| \cdot)$ is some discrepancy function between distributions and $(s,a)|s'$ is a time-reserved conditional probability, it is the conditional distribution of $(s,a)$ given that their next state is $s'$ following policy $\mu$.
\end{theorem}

\begin{proof}
According to the definition of $d^{\pi}$, we have
\begin{align*}
d^{\pi}\left(s^{\prime}\right) &=(1-\gamma) \sum_{t=0} \gamma^{t} d^{\pi}_{t}\left(s^{\prime}\right) \\
&=(1-\gamma) \rho\left(s^{\prime}\right)+(1-\gamma) \sum_{t=1}^{\infty} \gamma^{t} d^{\pi}_{t}\left(s^{\prime}\right) \\
&=(1-\gamma) \rho\left(s^{\prime}\right)+(1-\gamma) \gamma \sum_{t=0}^{\infty} \gamma^{t} d^{\pi}_{t+1}\left(s^{\prime}\right) \\
&=(1-\gamma) \rho\left(s^{\prime}\right)+(1-\gamma) \gamma \sum_{t=0}^{\infty} \gamma^{t} \sum_{s} P_{\pi}\left(s^{\prime} | s\right) d^{\pi}_{t}(s) \\
&=(1-\gamma) \rho\left(s^{\prime}\right)+\gamma \sum_{s} P_{\pi}\left(s^{\prime} | s\right)\left((1-\gamma) \sum_{t=0}^{\infty} \gamma^{t} d^{\pi}_{t}(s)\right) \\
&=(1-\gamma) \rho\left(s^{\prime}\right)+\gamma \sum_{s} P_{\pi}\left(s^{\prime} | s\right) d^{\pi}(s) \\
&=(1-\gamma) \rho\left(s^{\prime}\right)+\gamma \sum_{s, a} P\left(s^{\prime} | s, a\right) \pi(a | s) d^{\pi}(s).
\end{align*}

Then we have
\begin{align*}
d^{\mu}\left(s^{\prime}\right) \frac{d^{\pi}\left(s^{\prime}\right)}{d^{\mu}\left(s^{\prime}\right)} &= (1-\gamma) \rho\left(s^{\prime}\right)+\gamma \sum_{s, a} P\left(s^{\prime} | s, a\right) \pi(a | s) d^{\mu}(s) \frac{d^{\pi}(s)}{d^{\mu}(s)} \\
&= (1-\gamma) \rho\left(s^{\prime}\right)+\gamma \sum_{s, a} P\left(s^{\prime} | s, a\right) \pi(a | s) d^{\mu}(s) \frac{d^{\pi}(s)}{d^{\mu}(s)} \\
&= (1-\gamma) \rho\left(s^{\prime}\right)+\gamma \sum_{s, a} P\left(s^{\prime} | s, a\right) \mu(a | s) \frac{\pi(a | s)}{\mu(a | s)} d^{\mu}(s) \frac{d^{\pi}(s)}{d^{\mu}(s)}.
\end{align*}

Summing both sides over $s^{\prime}$, we get
\begin{align*}
\sum_{s^{\prime}} d^{\mu}\left(s^{\prime}\right) \frac{d^{\pi}\left(s^{\prime}\right)}{d^{\mu}\left(s^{\prime}\right)} = \sum_{s^{\prime}} (1-\gamma) \rho\left(s^{\prime}\right)+\gamma \sum_{s, a, s^{\prime}} P\left(s^{\prime} | s, a\right) \mu(a | s) \frac{\pi(a | s)}{\mu(a | s)} d^{\mu}(s) \frac{d^{\pi}(s)}{d^{\mu}(s)}.
\end{align*}

The above equation is equivalent to
\begin{align*}
\mathbb{E}_{s^{\prime} \sim d^{\mu}} \left[ \frac{d^{\pi}\left(s^{\prime}\right)}{d^{\mu}\left(s^{\prime}\right)} \right] &= \mathbb{E}_{s' \sim \rho} \left[ 1-\gamma \right] + \gamma \mathbb{E}_{s,a,s^{\prime} \sim d^{\mu}} \frac{\pi(a | s)}{\mu(a | s)} \left[ \frac{d^{\pi}(s)}{d^{\mu}(s)} \right] \\
&= (1-\gamma) + \gamma \mathbb{E}_{s,a,s^{\prime} \sim d^{\mu}} \left[ \frac{\pi(a | s)}{\mu(a | s)} \frac{d^{\pi}(s)}{d^{\mu}(s)} \right].
\end{align*}

Therefore, 
\begin{align*}
\mathbb{E}_{s,a,s^{\prime} \sim d^{\mu}} \left[  \frac{d^{\pi}\left(s^{\prime}\right)}{d^{\mu}\left(s^{\prime}\right)} - (1-\gamma) - \gamma \frac{\pi(a | s)}{\mu(a | s)} \frac{d^{\pi}(s)}{d^{\mu}(s)} \right] = 0.
\end{align*}

Denoting $w^{\pi}(s)=d^{\pi}(s) / d^{\mu}(s)$ and $\mathcal{T}^{\pi} w(s') := (1-\gamma) + \gamma \underset{(s,a)|s'}{\mathbb{E}} \left[ \frac{\pi(a|s)}{\mu(a|s)} w(s) \right]$ then we can have theorem 1.

\end{proof}

\section{Implementation Details}
All experiments are implemented with Tensorflow and executed on NVIDIA V100 GPUs.
For all function approximators, we use fully connected neural networks with RELU activations. For policy networks, we use tanh (Gaussian) on outputs. 
As in other deep RL algorithms, we maintain source and target Q-functions with an update rate 0.005 per iteration.
We use Adam for all optimizers. The batch size is 256 and $\gamma$ is 0.99.
We rescale the reward to $[0,1]$ as $r^{\prime}=\left(r-r_{\min }\right) /\left(r_{\max }-r_{\min }\right)$, where $r_{\max}$ and $r_{\min}$ is the maximum and the minimum reward in the dataset, note that any affine transformation of the reward function does not change the optimal policy of the MDP. 
Behavior, Actor, critic and density ratio network $w$ are all 3-layer MLP with 256 hidden units each layer. 
The learning rate of actor and behavior network is $1e-5$, the learning rate of critic and density ratio network is $1e-4$.
We search $\alpha$ in $\{0.2, 0.5, 2.0, 5.0\}$. We clip $\log w$ by 2.0 to avoid numerically instability.

\end{document}